\documentclass{article}

\usepackage[utf8]{inputenc} 
\usepackage[T1]{fontenc}    
\usepackage{hyperref}       
\usepackage{url}            
\usepackage{booktabs}       
\usepackage{amsfonts}       
\usepackage{nicefrac}       
\usepackage{microtype}      
\usepackage{xcolor}         
\usepackage[]{algorithmic} 
\usepackage{algorithm}
\usepackage{multirow}
\usepackage{amsmath} 
\usepackage{graphicx}
\usepackage{amsthm}
\usepackage{color}
\usepackage{soul}
\usepackage{xcolor,cancel}
\usepackage{authblk}

\newtheorem{theorem}{Theorem}[section]
\newtheorem{prop}[theorem]{Proposition}
\newtheorem{problem}{Problem}
\newtheorem{remark}[theorem]{Remark}
\newtheorem{lemma}[theorem]{Lemma}
\newtheorem{corollary}[theorem]{Corollary}
\newtheorem{assumption}[theorem]{Assumption}

\newcommand{\Var}{\mathrm{Var}}

\newcommand{\R}{\mathbb{R}}

\newcommand{\EE}{\mathbb{E}}

\newcommand{\PP}{\mathbb{P}}

\newcommand{\RR}{\mathbb{R}}
\newcommand{\mcC}{\mathcal{C}}
\newcommand{\mcD}{\mathcal{D}}
\newcommand{\mcF}{\mathcal{F}}

\newcommand{\mcG}{\mathcal{G}}

\newcommand{\mcB}{\mathcal{B}}

\newcommand{\mcQ}{\mathcal{Q}}

\newcommand{\mcL}{\mathcal{L}}

\newcommand{\bfY}{\mathbf{Y}}
\newcommand{\bfX}{\mathbf{X}}
\newcommand{\bfx}{\mathbf{x}}
\newcommand{\bfw}{\mathbf{w}}
\newcommand{\bfW}{\mathbf{W}}
\newcommand{\bfy}{\mathbf{y}}
\newcommand{\bfQ}{\mathbf{Q}}
\newcommand{\bfc}{\mathbf{c}}
\newcommand{\bfa}{\mathbf{a}}
\newcommand{\bfone}{\mathbf{1}}
\newcommand{\bfZ}{\mathbf{Z}}

\newcommand{\probp}{{\mathbb{P}}}
\newcommand{\probq}{{{Q}}}

\newcommand\abs[1]{\left|#1\right|}
\newcommand\Ep[1]{\mathbb{E}_\mathbb{P} \left[#1\right]}
\newcommand\Eq[1]{\mathbb{E}_\mathbb{Q} \left[#1\right]}

\newcommand{\rn}[2]{\frac{\mathrm{d}#1}{\mathrm{d}#2}}

\newcommand{\Linfty}{L^\infty}
\newcommand{\Lone}{L^1}
\newcommand{\Lzero}{L^0}

\newcommand{\rcondinfty}[1]{\rho\left(#1\right)}

\title{Multivariate Systemic Risk Measures and  Computation by Deep Learning Algorithms}

\author[1]{A. Doldi}
\author[2]{Y. Feng}
\author[3]{J.-P. Fouque\thanks{Corresponding Author: fouque@pstat.ucsb.edu}}
\author[4]{M. Frittelli}
\affil[1,4]{Università degli Studi di Milano}
\affil[2,3]{University of California, Santa Barbara}

\begin{document}

\maketitle

\begin{abstract}
\noindent In this work we propose deep learning-based algorithms for the computation of systemic shortfall risk measures defined via multivariate utility functions. We discuss the key related theoretical aspects, with a particular focus on the fairness properties of primal optima and associated risk allocations. 
The algorithms we provide allow for learning primal optimizers, optima for the dual representation and corresponding fair risk allocations. We test our algorithms by comparison to a benchmark model, based on a paired exponential utility function, for which we can provide explicit formulas. We also show evidence  of convergence  in a case for which explicit formulas are not available.
\end{abstract}



\section{Introduction}
The axiomatic approach to Systemic Risk Measures (SRM) has received increasing attention in the recent literature, motivated by the need of capturing and managing the risk for complex systems of interacting entities. Given a system of $N$ agents, each endowed with a risky financial position, the overall initial exposure  can be modeled by an $N$-dimensional random vector $\bfX$. A natural question is then how to evaluate the total risk $\rho(\bfX)$ carried by the system $\bfX$. A detailed overview of the vast current literature on the topic can be found e.g. in \cite{BFFMB},\cite{biagini2020fairness}, \cite{doldi2021conditional} and we only recall here that two main streams can be identified in the research on the topic. These two (possibly overlapping) perspectives  take the suggestive names of "SRM of first aggregate-then-allocate type" and "SRM of first-allocate-then-aggregate type". They are distinguished by different approaches to the problems of aggregation and allocation in the evaluation of the overall risk $\rho(\bfX)$ associated to $\bfX$.  {\label{addcits}We mention \cite{CIM13}, \cite{KO16} for details on the first type of SRM, and \cite{AR20}, \cite{AKM21}, \cite{Drapeau}, \cite{FR17} for the second type}. A further distinction can be introduced regarding the type of allocation procedures one adopts: on the one hand, one might assume that the total amount securing the system is allocated according to information at initial time (i.e. adopting deterministic allocations), and on the other hand, one might allow for a terminal-time, scenario-dependent allocation.

The present paper is concerned with SRM of first-allocate-then-aggregate type, with random allocations, and of shortfall type. More precisely we consider as our primal problem the minimization
$$\rho _B(\bfX) =\inf\left\{ \sum\limits_{n=1}^N Y^n \mid \bfY\in \mathcal{C},\, \Ep{U(\bfX+\bfY)}\,\geq\, B \right\}$$
where $U$ is a multivariate utility function in the form \eqref{eqn:U} below, and $\mcC$ is a suitable collection of \emph{feasible allocations} all sharing the following feature: although the components of $\bfY\in\mcC$ are random, their componentwise sums are (a.s.) deterministic.

Based on the technical findings in \cite{doldi2021conditional} and \cite{doldi2022multivariate}, we provide some main properties for $\rho_B$ above. Along the lines of \cite{biagini2020fairness}, we discuss some \emph{fairness} properties for optimal allocations and corresponding risk allocations associated to $\rho_B$. We discuss a Nash-equilibrium type fairness, since using a multivariate utility function for the system induces preferences for each agent that \emph{depend on the action of other agents in the system}. 
 We conclude our discussion on theoretical aspects providing explicit formulas for all quantities of interest (primal and dual optima, the value of the systemic risk measure $\rho_B(\bfX)$, the form of the penalty function in the dual representation (see Problem \ref{dualprob} and Theorem \ref{thm:recap} Item 3 below for details, and section \ref{equilriskalloc} for interpretations) in a paired exponential setup.

 The second main part of our work deals with the problem of approximating primal and dual optima for the risk measure $\rho_B$ briefly described above, extending the deep learning techniques in \cite{feng2022deep} to the present  framework which allows for multivariate utilities. Regarding  the application of neural networks for approximating
the solutions to high-dimensional optimization problems, we also refer to \cite{feinstein2022deep} and \cite{onken2022neural}. Deep learning establishes an automatic algorithm to improve performance of tasks and thus has been applied to many areas; see for instance \cite{min2021signatured} for solving stochastic differential games.
We refer the reader to the latter two papers for an extended discussion on current literature, and we limit ourselves to citing \cite{goodfellow2020generative} for a GAN method, based on
solving a min-max problem, which shed some lights on our algorithm design in this paper. 

In section \ref{section3:alg} we propose Algorithm \ref{algo:primal} for approximating optima of the primal problem, namely optimizers in the definition of $\rho_B$, and Algorithm \ref{algo:dual} for approximating optima for the dual representation of $\rho_B(\bfX)$. Knowing (how to approximate) such optima is crucial for the evaluation of the risk allocations we mentioned before. Both the primal and dual algorithms we propose also approximate the value of $\rho_B(\bfX)$, which allows for comparisons and sanity checks.
In section \ref{sec:PairedNormalgroupExp} and section \ref{sec:PairedUBetagroupExp} we test, with good outcomes, our algorithms using the explicit formulas for the paired exponential utility as a benchmark, with laws of $\bfX$ based on Gaussian and Beta destributions respectively.
Section \ref{sec:OtherUNormalgroupExp}  treats a general case where formulas are not available. As an additional sanity check for our findings, we investigate if our candidate approximate optima can be obtained as a function of  the componentwise sum $\sum_{n=1}^NX^j$, since such a dependence is predicted by the theory. In both the cases of known and unknown explicit formulas we present evidence of a satisfactory behaviour for our approximation algorithms.

\subsection{Notation}
We fix an underlying probability space $(\Omega,\mcF,\probp)$, and the usual associated Lebesgue spaces for $p\in\{0\}\cup [1,\infty]$: $L^p=L^p(\Omega,\mcF,\probp)$. Throughout the paper, we use the extended notation in RHS only when the underlying measure is changed.  We use bold notation for vectors, be those deterministic vectors, random vectors or vectors of (probability) measures. For a (probability) measure on $(\Omega,\mcF)$ we write $Q\ll\probp$ for "$Q$ is absolutely continuous w.r.t. $\probp$", and denote by $\rn{Q}{\probp}\in \Lone$ the corresponding Radon-Nikodym derivative. For vectors of probability measures $\bfQ$ we write $\bfQ\ll\probp$ for componentwise absolute continuity. We set $$\mcQ:=\{\text{vectors }\bfQ=(Q^1,\dots,Q^N)\text{ of probability measures on }\mcF,\text{ s.t. } \bfQ\ll \probp\}$$ and denote by $\rn{\bfQ}{\probp}$ the vectors of Radon-Nikodym derivatives of $\bfQ\in\mcQ$. We write $\mathbf{1}=[1,\dots, 1]$. 

\section{Problem Setup}
Inspired by \cite{doldi2022multivariate}, we extend our deep learning techniques in \cite{feng2022deep} to a general setting where the multivariate utility function takes the form of 
\begin{equation}\label{eqn:U}
    U(\mathbf{x}):=\sum\limits_{n=1}^N u_n(x^n)+\Lambda(\mathbf{x}),\quad \mathbf{x}\in \mathbb{R}^N.
\end{equation}
Here, $u_1,\ldots,u_N : \mathbb{R} \to \mathbb{R}$ are univariate utility functions and
$\Lambda : \mathbb{R}^N \to \mathbb{R}$ is concave, increasing 
and bounded from above (we defer  to Assumption \ref{safeassumtpion} below for details). The multivariate utility function $U(\mathbf{x})$ will be taken to fulfill all requirements in \cite{doldi2021conditional}, detailed below, which play the same role as the Inada conditions in the univariate case.
We mention as a key example the following multivariate paired exponential utility function
\begin{equation}
\begin{split}
\label{utilalternative}
U(x)&=\frac{N^2}{2}-\frac{1}{2}\left(\sum_{n=1}^N\exp\left(-\alpha_nx^n\right)\right)^2 \\
&=\frac{1}{2}\sum\limits_{n=1}^N (1-e^{-2\alpha_nx^n})+\frac{1}{2}\sum\limits_{n,m=1; n\neq m}^N(1-e^{-(\alpha_n x^n+\alpha_mx^m)})
\end{split}
\end{equation} 
which will have a prominent role in the following, since it allows for obtaining explicit formulas.

Using a multivariate utility in the form \eqref{eqn:U}, we introduce a systemic risk measure of the "first-allocate-then-aggregate type" following \cite{biagini2020fairness}, allowing for random allocations: we consider the set of random vectors whose componentwise sums are (a.s.) deterministic, namely
\begin{align*}
&\mathcal{C}_\mathbb{R} := \left\{\mathbf{Y}\in (\Lzero)^N\mid\sum\limits_{n=1}^N Y^n\in \mathbb{R}\right\},
\end{align*} and  we take a set of \emph{feasible random allocations} $\mathcal{C}\subseteq \mathcal{C}_\mathbb{R}$. We refer to \cite{BFFMB} and \cite{biagini2020fairness} for a detailed motivation, discussion and some examples regarding the scenario-dependent allocation procedure. We defer the technical assumptions  regarding the set of feasible allocations to Assumption \ref{safeassumtpion}, and introduce the primal problem we aim at studying.
\begin{problem}[Primal] The primal formulation of the systemic risk measure, for $\bfX\in(\Linfty)^N$, is given by
\begin{equation} 
\rho_B(\bfX)\,:=\,\inf\left\{ \sum\limits_{n=1}^N Y^n \mid \bfY\in \mathcal{C}\cap (\Linfty)^N,\, \Ep{U(\bfX+\bfY)}\,\geq\, B \right\}.
\label{eqn:defrho}
\end{equation}
\label{problem:primal}
\end{problem}
We stress again that each $Y^n$, $n=1,\ldots,N$, is random and depends on the scenario $\omega$ realized at terminal time $T$, but the sum of random allocations is deterministic and known at the beginning, i.e. $\sum_{n=1}^N Y^n\in \mathbb{R}$. Thus the \textit{overall systemic risk} $\rho(\bfX)$ can be interpreted as the minimal total cash amount needed today to secure the system by distributing the cash at the future time $T$ among the components of the risk vector $\bfX$. However, it is also important to know how much each financial institution of the system contributes to the overall systemic risk. We call a  \emph{risk allocation} of each financial institution $n$ some deterministic amount $\rho^n(\mathbf{X})\in \mathbb{R}$ satisfying the "Full Allocation" property, see \cite{brunnermeier2019measuring} for example, which reads \[\sum_{n=1}^N\rho^n(\bfX) = \rho(\bfX)\,.\] After a discussion on the main properties of $\rho_B(\cdot)$ and following \cite{biagini2020fairness}, we will provide in section \ref{equilriskalloc} a natural risk allocation procedure and argue \emph{fairness} for such allocations from different angles: we will explain how such procedure, also in this new multivariate setup, can be considered satisfactory both from the overall system's point of view and from the perspective of each single institutions in it.
Such a discussion and interpretation are based on convex duality, which will allow us to obtain the following result.

\begin{problem}[Dual]
\label{dualprob} The dual representation of the systemic risk measure in Problem \ref{problem:primal} is given by
\[ \rho _B(\bfX)= \max_{\mathbf{Q}\in \mathcal{D}}\left\{\sum_{n=1}^N\mathbb{E}_{Q^n}[-X^n]-\alpha_B(\mathbf{Q})\right\}\]
\label{rho}
where $\alpha_B$ and $\mcD$ are given below in \eqref{eqn:alpha} and \eqref{eq:dual_restriction_on_Q} respectively.
\label{problem:dual}
\end{problem}

\section{Theoretical aspects}
\label{sec:theo}
\subsection{Properties of $\rho_B$ and dual representation}
This section collects some key results regarding the shortfall type systemic risk measure we introduced in \eqref{eqn:defrho}. We greatly rely on \cite{doldi2021conditional}, which covers our setup. 
We work under the following assumptions, which are rather technical but satisfied in all the examples and cases we consider in the following.
\begin{assumption}
\label{safeassumtpion}
$\,$

\begin{itemize}
\item $B\in
\R$ satisfies $B<\sup_{z\in{\mathbb{R}}^N}U(z)\leq
+\infty$.
\item Regarding $U$:
\begin{itemize}
\item 
$U:\R^N\rightarrow \R$ takes the form \eqref{eqn:U} where $u_{1},\dots ,u_{N}:{\mathbb{R}}\rightarrow {\mathbb{R}}$ are
strictly concave, strictly increasing and $\Lambda :{\mathbb{R}}^{N}\rightarrow {%
\mathbb{R}}$ is (possibly non strictly) concave and increasing with respect to the partial
componentwise order, and bounded from above. 
For every $\varepsilon >0$ there exist a point $z_\varepsilon\in{%
\mathbb{R}}^N$ and a selection $\nu_\varepsilon\in\partial
\Lambda(z_\varepsilon)$, such that $\sum_{j=1}^{N}\left\vert
\nu_\varepsilon\right\vert <\varepsilon\,.$
For each$\,j=1,\dots ,N\,,$ we also assume the Inada conditions 
\begin{equation*}
\lim_{x\rightarrow +\infty }\frac{u_{j}(x)}{x}=0\,\,\,\,\,\,\,\text{ and }%
\,\,\,\,\,\,\,\lim_{x\rightarrow -\infty }\frac{u_{j}(x)}{x}=+\infty
\end{equation*}%
and that, without loss of generality, $u_{j}(0)=0$. 
    \item 
    
If $\bfZ\in (\Lone)^{N}$ satisfies $(U(\bfZ))^{-}\in
\Lone$ then there exists $\delta >0$
s.t. $(U(\bfZ-\varepsilon \bfone))^{-}\in \Lone$ for all $ 0\leq \varepsilon <\delta$.
    \item 
    If for $\bfZ\in (L^{0}\left( (\Omega ,\mathcal{F},{%
\mathbb{P}});[-\infty ,+\infty ]\right) )^{N}$ there exist $\lambda
_{1},\dots ,\lambda _{N}>0$ such that $\Ep{
u_{j}(-\lambda _{j}\left\vert Z^{j}\right\vert )} >-\infty $ for every $j=1,\dots,N$, then
there exists $\alpha >0$ such that $\Ep{ \Lambda
(-\alpha \left\vert \bfZ\right\vert )} >-\infty $
\end{itemize}
    \item Regarding $\mcC$: 
    $\mcC\subseteq \mcC_\R$ is a convex cone closed in probability. Moreover $\mcC+ \R^N=\mcC$ and  $\mcC$ is closed under truncation: $\forall
\bfY\in \mcC$ there exists $k_{\bfY}\in \mathbb{N}$ and a $\bfc_{\bfY}\in
\R^N$ such that $\sum_{j=1}^{N}c_{Y}^{j}=%
\sum_{j=1}^{N}Y^{j}$ and $\forall k\geq k_{\bfY},\,k\in \mathbb{N}$ 
\[
\bfY_{(k)}:=\bfY1_{\bigcap_{j}\{\left\vert Y^{j}\right\vert <
k\}}+\bfc_{\bfY}1_{\bigcup_{j}\{\left\vert Y^{j}\right\vert \geq k\}}\in \mcC. \]
\end{itemize}
\end{assumption}

We  collect in the following statement all we need from the technical point of view: existence and uniqueness of primal optima, a dual representation and existence of dual optima.

\begin{theorem}
\label{thm:recap}
Under Assumption \ref{safeassumtpion}

\begin{enumerate}
\item {$\rho_B:(\Linfty)^N\rightarrow \R$ is a systemic convex risk measure, namely $\rho_B$ is monotone (w.r.t the componentwise a.s. order), convex and monetary (i.e. $\rho_B(\bfX+\mathbf{m})=\rho_B(\bfX)-\sum_{n=1}^Nm^n$ for all $\mathbf{m}\in\R^N$). It is also
continuous from above and from below.}
\item{For every $X\in (\Linfty)^{N}$ 
\begin{equation}
\label{eqn:foroptimumprimal}
\rho _B(\bfX) =\inf\left\{ \sum\limits_{n=1}^N Y^n \mid \bfY\in \mathcal{C}\cap (\Lone)^N,\, \Ep{U(\bfX+\bfY)}\,\geq\, B \right\}
\end{equation}}

and the essential infimum in RHS is attained by a unique  $\bfY_\bfX\in\mathcal{C}\cap (\Lone)^N$.

\item $\rho_B$ admits the following dual
representation:%
\begin{equation}
\rho_B(\bfX) =\max_{{\bfQ}\in %
\mcD}\left( \sum_{j=1}^{N}\mathbb{E}_{{Q}%
^{j}}\left[ -X^{j}\right] -\alpha_B(\bfQ)\right)\quad \forall \,\bfX\in (\Linfty)^{N}
\label{DynRMeqdualreprshorfall}
\end{equation}%
where for every vector $\bfQ=(Q^1,\dots,Q^N)\in\mcQ$
\begin{align}
& \alpha_B(\bfQ):= \sup\left\{ \sum_{n=1}^N\mathbb{E}_{Q^n}[-Z^n]\mid \mathbf{Z}\in(\Linfty)^N, \Ep{U(\mathbf{Z})}\geq B\right\},  \label{eqn:alpha} \\
 &\mcD:=\left\{\bfQ\in\mcQ\mid \alpha_B(\bfQ)<+\infty, 
\sum_{j=1}^N\mathbb{E}_{\mathbb{Q}^j}
\left[Y^j\right]\leq \sum_{j=1}^NY^j,\,\forall
\bfY\in\mcC\cap(\Linfty)^N\right\} .\label{eq:dual_restriction_on_Q}
\end{align}
In the particular case $\mcC=\mcC_\R$ we have 
\begin{equation}
    \label{polarity}
\mcD=\left\{\bfQ\in\mcQ\mid \alpha_B(\bfQ)<+\infty, \text{ } 
Q^i=Q^j\,\text{ on } \mcF \text{ } \forall
i,j\in\{1,\dots,N\}\right\} 
\end{equation}
\item For the unique optimum $\bfY_\bfX$ in RHS of \eqref{eqn:foroptimumprimal} and for any $\bfQ\in \mcD$ it holds that  $Y^j_\bfX\in L^{1}((\Omega ,\mathcal{F},{\mathbb{Q}}^{j}),j=1,\dots,N$ and 
\[\sum_{j=1}^{N}\mathbb{E}_{\mathbb{Q}^{j}}\left[ Y_\bfX^{j}\right] \leq \sum_{j=1}^{N}Y_\bfX^{j}\,.\]
Moreover, setting
\begin{equation}
\label{defLcal}
    \mcL=\mcL^1\times\dots\times\mcL^N:=(\Lone)^{N}\cap \bigcap_{\bfQ\in \mcD} L^{1}(\Omega,\mcF,Q^1)\times\dots\times L^{1}(\Omega,\mcF,Q^N)\,,
\end{equation}
we have:  
$\bfY_\bfX\in \mcL$, $\Ep{ U\left( \bfX+\bfY_\bfX\right)} \geq B$ and for any optimum 
$\bfQ_\bfX$ of \eqref{DynRMeqdualreprshorfall},
\begin{equation}
\label{eq:optimaareOK}
\rho _B(\bfX)=\sum_{n=1}^{N}\mathbb{E}_{\bfQ_\bfX^n}\left[ Y^n_\bfX\right]=\min\left\{
\sum_{n=1}^{N}\mathbb{E}_{Q_\bfX^{n}}\left[ Y^{n}\right] \mid \mathbf{Y}\in \mcL,\,\Ep{ U\left( X+Y\right)} \geq B\right\}  \,.
\end{equation}%

\item
 Suppose additionally that $u_1,\dots,u_N$ are bounded from above, and that $U$ is differentiable. Then the optimum for \eqref{DynRMeqdualreprshorfall} is unique in $\mcD$, and if $\mcC=\mcC_\R$ both $\rn{\bfQ_\bfX}{\probp}$ and $\bfX+\bfY_\bfX$ are (essentially) $\sigma(X^1+\dots+X^N)$-measurable.
 \end{enumerate}
\end{theorem}
\begin{proof}
Set $\mcG=\{\emptyset,\Omega\}$. In this proof, we make reference to the results in \cite{doldi2021conditional}. The fact that $\rho_B$ is a systemic risk measure, continuity from above and below, existence of a primal optimizer, and the dual representation are obtained from  Theorem 5.4. Uniqueness of primal optimizers comes from Proposition 5.11. The claims in  Item 4 of the statement derive from Proposition 5.13 and Theorem  5.14 (observe that Assumption \ref{safeassumtpion} in this paper absorbs both Assumptions 5.10 and  5.12 in \cite{doldi2021conditional}).
The particular form in \eqref{polarity} is obtained exactly as in \cite{biagini2020fairness} Proposition 3.1 Item (ii). Dual uniqueness is proved in Proposition \ref{propuniquedual}, as well as the claimed measurability of optima.\end{proof}
In particular, Item 3 of Theorem \ref{thm:recap} proves the duality stated in Problem \ref{problem:dual}.
\begin{remark}
   \label{remmovetol1} 
We stress that \eqref{eqn:foroptimumprimal} entails the statement  that the value of $\rho_B$ is left unchanged if the intersection with  $(L^\infty)^N$ in the definition \eqref{eqn:defrho} is replaced by the one with $(L^1)^N$. This motivates the use of the term "optimum" or "minimum" for $\mathbf{Y}_\mathbf{X}$.
\end{remark}

\subsection{Equilibrium risk allocation}
\label{equilriskalloc}

We know from Theorem \ref{thm:recap} that, given a position $\bfX\in (\Linfty)^N$ and the primal and dual optima $\bfY_\bfX,\bfQ_\bfX$ for $\rho_B(\bfX)$, the amounts\begin{equation}\label{eqn:IndividualRiskAllocation}
    \rho^n(\bfX)\, :=\, \mathbb{E} _{{Q}^{n}_{\bfX}} [ Y^{n}_{\bfX} ]\,\quad \text{for}\quad n =1,\ldots,N,
 \end{equation} are well defined. 
  We stress also that only for a particular class of utility functions, stated in section \ref{sec:PairedNormalgroupExp} below,  explicit formulas for such optima are available, which justifies the approximation techniques adopted in the second part of this paper.

These amounts 
are natural candidates as risk allocations, as  \eqref{eq:optimaareOK} shows that the full allocation property for the risk allocations given by \eqref{eqn:IndividualRiskAllocation} holds. 
Arguing verbatim as in \cite{biagini2020fairness}, \eqref{eq:optimaareOK} "shows that the valuation by $\mathbb{E}_{ \bfQ_\bfX}[\cdot]$ agrees with the systemic risk measure $\rho_B(\bfX)$. This supports the introduction of $\mathbb{E}_{ \bfQ_\bfX}[\cdot]$ as a suitable systemic valuation operator." 
  
  We now motivate why such a choice is fair for the agents in the system. 
 We follow in spirit \cite{biagini2020fairness}, but some conceptual and technical modifications are in place, to take into account the multivariate nature of the general utility functional $U$ we are adopting in this work.

First, one can argue in favour of the choice of the triple $(\bfY_\bfX,(\rho^n(\bfX))_n,\bfQ_\bfX)$ by recalling that it constitutes a suitably defined equilibrium, the Multivariate Systemic Optimal Risk Transfer Equilibrium (mSORTE, see \cite{doldi2022multivariate} for a detailed description).
\begin{theorem}[\cite{doldi2021conditional} Theorem 7.3]
\label{thm:msorte}
Suppose Assumption \ref{safeassumtpion} holds. Let $\bfX\in (\Linfty)^{N}$.
Let $\bfY_\bfX$ be the unique primal optimum from Theorem \ref{thm:recap}, and let $\bfQ_\bfX\in\mcD$ be a dual optimum in \eqref{DynRMeqdualreprshorfall}. Define $a_\bfX^{n}:=\mathbb{E}_{Q_\bfX^n}[Y_\bfX^n]$ for $n=1,\dots,N$. Then $(\bfY_\bfX,\bfa_\bfX,\bfQ_\bfX)$ is a mSORTE with budget $\rho_B(\bfX)$, namely:
\begin{enumerate}
\item $(\bfY_\bfX,\bfa_\bfX)$ is an optimum for 
\begin{equation}
\sup_{\substack{ \bfa\in\R^N  \\ %
\sum_{j=1}^{N}a^n=\rho_B(\bfX)}}\left\{ \sup\left\{ \Ep{
 U(\bfX+\bfY)} \mid \bfY\in \mcL,\,\mathbb{E}_{Q_\bfX^{n}}\left[ Y^{n}
\right] \leq a^n,\,\forall n\right\} \right\} \,;
\label{DynRMeqmsorteparticular}
\end{equation}

\item $\bfY_\bfX\in\mcC$ and $\sum_{n=1}^{N}Y_{\bfX}^{n}=\rho_B(\bfX)\, {\mathbb{P}}$-a.s..
\end{enumerate}
    
\end{theorem}
Theorem \ref{thm:msorte} allows for the following discussion:  $\bfY_\bfX$ is an optimum for 
\begin{equation}
\label{maxsyst}
 \max\left\{ \Ep{
 U(\bfX+\bfY)} \mid \bfY\in \mcL,\,\mathbb{E}_{Q_\bfX^{n}}\left[ Y^{n}
\right] \leq \mathbb{E}_{Q_\bfX^n}[Y_\bfX^n],\,\forall n\right\} .    
\end{equation}

If $U(\mathbf{x})=\sum\limits_{n=1}^N u_n(x^n)$ (which,   apart from technicalities such as the use of Orlicz Spaces, corresponds to the setup in \cite{biagini2020fairness} with $\Lambda=0$), we could easily infer from \eqref{maxsyst} that, for any $n=1,\dots,N$, $Y_\bfX^n$ is an optimum for the single agent's uitlity maximization problem 
\[\max\left\{\Ep{u_n(X^n+Y^n)}\mid Y^n\in\mcL^n \text{ }\mathbb{E}_{Q_\mathbf{X}^n}[Y^n] \leq  \mathbb{E}_{Q_\mathbf{X}^n}[Y_\mathbf{X}^n]\right\}.\]
This in \cite{biagini2020fairness} allowed for the claim that "the optimal allocation $Y_\bfX^n$
and its value $\mathbb{E}_{Q_\mathbf{X}^n}[Y_\mathbf{X}^n]$ can thus be considered fair by
the $n$-th bank as  $Y^n_\bfX$ maximises its individual expected utility among all random
allocations (not constrained to be in $\mcC_\R$) with value $\mathbb{E}_{Q_\mathbf{X}^n}[Y_\mathbf{X}^n]$. [...] This finally argues for the fairness of the risk allocation
$(\mathbb{E}_{Q_\mathbf{X}^1}[Y_\bfX^1],\dots, \mathbb{E}_{Q_\mathbf{X}^N}[Y_\bfX^N])$ as fair valuation of the optimal scenario-dependent allocation $\bfY_\bfX$."
The presence of a possibly nontrivial $\Lambda$ here does not allow for a direct adaptation of such an argument, in that one cannot in general split the optimization \eqref{maxsyst} as single agent's utility maximization problems. However, something conceptually similar can be carried over.

To this end, we introduce the following compact notation: for a vector $\bfx\in\R^N$ and $z\in\R$ we write $[\bfx^{[-j]},z]\in \R^N$ for the $N-$dimensional vector obtained substituting $z$ to the $j-$th component of $\bfx$, $j=1,\dots,N$. The identical notation is used for random vectors.
Inspired by the Nash Equilibrium type property in \cite{doldi2022multivariate} Theorem 4.5, we can state the following easy corollary:
\begin{corollary}
\label{cor:nash}
With the same assumptions and notation as in Theorem \ref{thm:msorte}, for every $n=1,\dots,N$, $Y_\bfX^n$ is optimal for the  maximization problem
\begin{equation}
\label{eqn:nash}
    \sup\left\{\Ep{U\left(\bfX+[\bfY_\bfX^{[-n]},Y]\right)}\mid Y\in \mcL^n, \mathbb{E}_{Q_\bfX^{n}}\left[ Y
\right] \leq \mathbb{E}_{Q_\bfX^n}[Y_\bfX^n]\right\}\,.
\end{equation}
\end{corollary}%
The a priori surprising fact that the optimum for \eqref{eqn:nash} actually produces the $n-$th component  of $\bfY_\bfX$ and that such an optimum yields an element of $\mcC$ (which requires $\sum_{n=1}^NY_\bfX^n\in\R$) is a consequence of the particular choice $Q^n_\bfX$ as (equilibrium) pricing measure.

We can still argue 
that the optimal allocation $(\bfY_\bfX,\bfa_\bfX)$ can thus be considered fair by
the $n-$th bank in the system. This time, however, an equilibrium is at play: given any vector $\bfW\in\mcL$, $\Ep{U\left(\bfX+[\bfW^{[-n]},Y]\right)}$ in \eqref{eqn:nash} stands for the expected utility for agent $n$, \emph{given that all the other agents in the system adopt the strategy $(W^1,\dots,W^{n-1},W^{n+1},\dots,W^N)$}. 
Let us now look at \eqref{eqn:nash}, for a fixed $n$. Supposing that all other agents but the $n$-th one adopt the schenario-dependent allocations $Y_\bfX^j$ and the corresponding value $ \mathbb{E}_{Q_\bfX^j}[Y_\bfX^j] ,j\neq n$:
\begin{itemize}
\item the risk allocation $\mathbb{E}_{Q_\bfX^n}[Y_\bfX^n]$ naturally arises from  the full allocation property as by \eqref{eq:optimaareOK} \[\mathbb{E}_{Q_\bfX^n}[Y_\bfX^n]=\rho_B(\bfX)-\sum_{j\neq n}\mathbb{E}_{Q_\bfX^j}[Y_\bfX^j].\]
\item the scenario-dependet allocation $Y_\bfX^n$ naturally arises by optimality of $\bfY_\bfX$ as \[Y_\bfX^n=\rho_B(\bfX)-\sum_{j\neq n}Y_\bfX^j\]
\item even if agent $n$ were allowed to look for other scenario-dependent allocations $Y^n$ among all those variables whose initial "price" does not exceed the budget $\mathbb{E}_{Q_\bfX^n}[Y_\bfX^n]$ (with such $Y^n$ not necessarily satisfying the terminal-time full allocation requirement that $Y^n+\sum_{j\neq n}Y_\bfX^j=\rho_B(\bfX)$), the agent's expected utility \emph{given the others' choices} could not be improved above the one attained adopting the allocation $Y_\bfX^n$.
\end{itemize}
The optimal allocation $Y^n_\bfX$
 and its value $\mathbb{E}_{Q_\bfX^n}[Y_\bfX^n]$ can thus be considered fair by
the $n-$th agent as $Y^n_\bfX$
 maximises the individual expected utility, \emph{provided all agents but $n$ act according to $\bfY_\bfX^{[-n]}$}, among all random
allocations with values not exceeding the budget $\mathbb{E}_{Q_\bfX^n}[Y_\bfX^n]$.

\subsection{Explicit formulas for the paired exponential utility function}
We conclude the first part of the work with an example, providing explicit formulas for all the quantities descried in the previous sections, given a paired exponential  (multivariate) utility function $U$ in the form \eqref{utilalternative}

\begin{theorem}
\label{thmexplicitformulas}
Take $\bfX\in(\Linfty)^N$, $\mathcal{C}=\mathcal{C}_\R$, $\alpha
_{1},\dots ,\alpha _{N}>0$, $U$ as in \eqref{utilalternative} and let $B<\sup_{z\in{\mathbb{R}}^N}U(z)=\frac{N^2}{2}$. 
Set 
\begin{equation}
\begin{gathered}
\beta :=\sum_{j=1}^{N}\frac{1}{\alpha _{j}};\quad
\Gamma:=\sum_{j=1}^N\frac{1}{\alpha_j}\log\left(\frac{1}{\alpha_j}\right);\quad \widehat{\lambda}(B):=-\frac{2}{\beta}\left(B-\frac{N^2}{2}\right)=\frac{N^2-2B}{\beta}>0;\\
S:=\sum_{j=1}^{N}X^{j};\quad d(\mathbf{X}) = \frac{\beta}{2}\log\left(\frac{\beta^2}{N^2-2B}\,\Ep{\exp\left(-\frac{2S}{\beta}\right)}\right) \in \mathbb R.
\end{gathered}
\label{defconstants}
\end{equation}
Then the value $\rho_B(\bfX)$, the  primal optimum $\bfY_\bfX \in L^\infty$, the dual optimum $\bfQ_\bfX\in\mcD$ and the fair risk allocations in \eqref{eqn:IndividualRiskAllocation} take the following forms.

\begin{align}
\rho_B(\mathbf{X})&= d(\mathbf{X})-\Gamma,\label{eqn:optrho}
\\
Y_\bfX^n(\omega)& = -X^n(\omega)+\frac{1}{\beta\alpha_n}(S(\omega)+d(\mathbf{X}))-\frac{1}{\alpha_n}\log(\frac{1}{\alpha_n}),\label{eqn:optY}
\\
\rn{Q_\bfX^1}{\probp}&=\dots=\rn{Q_\bfX^N}{\probp}=\rn{Q_\bfX}{\probp}= \dfrac{\exp\left(-2S/\beta\right)}{\mathbb{E}\left[\exp\left(-2S/\beta\right)\right]},\label{eqn:optQ}\\
\rho^n(\mathbf{X}) &= \mathbb{E}_{Q_\bfX}[Y_\bfX^n] = \mathbb{E}_\probp\left[Y_\bfX^n\cdot\frac{dQ_\bfX}{d\mathbb{P}} \right] .\notag
\end{align}
Furthermore,  for every $\bfQ\in\mcD$ (so that $\bfQ=(Q,\dots,Q)$) the penalty function in \eqref{eqn:alpha} is given by
$$\alpha_B(\bfQ)=\left(\Gamma-{\frac{\beta\log(\beta)}{2}}\right)+\frac{\beta}{2}\log\left({\widehat{\lambda}(B)}\right)+\frac{\beta}{2}H(\probq|\probp)$$
where $H(Q|\probp)=\Ep{\rn{\probq}{\probp}\log\rn{\probq}{\probp}}$ for $Q\ll\probp$ is the  relative entropy.
\end{theorem}

\begin{proof}
    See Appendix \ref{appendix:expsol}.
\end{proof}

\section{Algorithms}
\label{section3:alg}

In this section, we extend the deep learning algorithms in \cite{feng2022deep} to the present framework, which allows for multivariate utilities, in order to approximate the solutions of the primal problem and dual problem. Since the dual problem involves optimization over the space of probability measures, we design a neural network to represent the Radon-Nikodym derivative which can be applied to the general change of measure problem from physical measure to risk neutral measure. 

In the following sections, we work under the assumptions of Theorem \ref{thm:recap} and the additional assumptions of its item 5 so that we have uniqueness of the solutions of the primal and dual problems,  and, in particular, thanks to $\mcC=\mcC_\R$, both $\rn{\bfQ_\bfX}{\probp}$ and $\bfX+\bfY_\bfX$ are (essentially) $\sigma(X^1+\dots+X^N)$-measurable. This last property justifies our parametrization $\mathbf{Y}= \varphi(\mathbf{X})$ in the primal problem in section \ref{sec:primal} and $\frac{d Q}{d \PP}=\Theta(\mathbf{X})$ in the dual problem in section \ref{sec:dual} for deterministic functions $ \varphi$ and $\Theta$. In either case, even though from the theoretical point of view we know that $\bfX+\bfY_\bfX$ and $\rn{\bfQ_\bfX}{\probp}$ should be functions of the sum $S=\sum_{j=1}^N X^j$, we do not impose this condition at a neural network level. Instead, we let the algorithms learn such a dependence and use it a posteriori as a sanity check, especially in cases where no explicit solution is available as in sections \ref{sec:OtherUNormalgroupExp} and \ref{sec:PairedUBetagroupExp}. Finally, we stress that the (optimal) functions $\varphi,\Theta$ are not universal, in that they depend on $\mathbf{X}$ through its law under $\probp$, as can be seen explicitly in \eqref{eqn:optY}.

\subsection{Primal Problem}\label{sec:primal}
We take $\varphi: \mathbb{R}^N \to \mathbb{R}^N$ to denote the fully connected neural networks parameterized by weights and biases $(w, b)$. They take the risk factor $\mathbf{X}(\omega)\in \mathbb{R}^N$ as input and generate cash allocation $\mathbf{Y}(\omega)\in \mathbb{R}^N$ as output for any scenario $\omega\in\Omega$. To ease the notation, we will omit $\omega$ in the rest of the paper when the context is clear. More precisely, 
$$\mathbf{Y} := (Y^1, \dots, Y^N) = (\varphi_{1}(\mathbf{X}), \dots, \varphi_{N}(\mathbf{X}) )=:\varphi(\mathbf{X}).$$

Based on the primal Problem \ref{problem:primal}, we first add a penalty for the variance of total cash allocation to the loss function since the problem requires the total cash added to the system to be deterministic. We then add a second penalty term to deal with the acceptance set restriction in the problem setup. As a result, the objective function for our deep learning task, designed for the primal problem, becomes
\begin{align}
\label{eq:rho_tilde}
    J_{\text{primal}}(\varphi) := \Ep{\sum_{i} \varphi_i(\mathbf{X})} & + \mu \cdot \Var\big(\sum_i \varphi_i(\mathbf{X})\big) \nonumber \\
    &+ \lambda \cdot \big(B-\EE_\probp\big[U(\mathbf{X}+\mathbf{Y}) \big]\big)^+, 
\end{align}
where $\mu, \lambda$ are hyperparameters and we write $$\tilde{\rho}(\mathbf{X}) = \inf_{w, b} J_{\text{primal}}(\varphi).$$

In \eqref{eq:rho_tilde}, the second term is the penalty for the variance and the third term is the penalty for failure of falling into the acceptance set. With the proper choice of hyperparameters $\mu$ and $\lambda$, the two penalties are very close to $0$ at optimal, which renders $\tilde{\rho}(\mathbf{X}) \approx \rho(\mathbf{X})$. In practice, we will compute the empirical estimation of $J_{\text{primal}}(\varphi)$ by using Monte Carlo algorithm to estimate the variance and expectation in \eqref{eq:rho_tilde}. The details are provided in Algorithm \ref{algo:primal}.

\begin{algorithm}
\caption{Primal problem.}
\label{algo:primal}
\begin{algorithmic}
\REQUIRE Data $\{\mathbf{X}(\omega_i)\}_{i=1}^{\text{batch}}$, neural net $\varphi$, function $U$ and hyperparameters $\mu, \lambda$, $B$, learning rate $\gamma$, Epochs

\FOR{$e=1$ to Epochs}
\STATE compute empirical estimation $\hat{J}$ of $J_{\text{primal}}(\varphi)$  by Monte Carlo
\STATE compute gradients $\nabla_w\hat{J}$ and $\nabla_b\hat{J}$
\STATE update $\varphi$: $w=w-\gamma\nabla_w\hat{J}$ and $b=b-\gamma\nabla_b\hat{J}$
\ENDFOR
\STATE compute empirical estimation: $\hat{\rho}=\hat{J}_{\text{primal}}(\varphi)$ 
\ENSURE Updated neural net $\varphi$, $\hat{\rho}$
\end{algorithmic}
\end{algorithm}

\subsection{Dual Problem}\label{sec:dual}
We  only consider the single group case $\mcC=\mcC_\R$ for the dual problem \ref{problem:dual} where all the measures $Q^n$ are the same by Theorem \ref{thm:recap}, and we  simply write $Q$ instead of $Q^n$ for all $n$ in the following discussion. That is, the polarity restriction in \eqref{eq:dual_restriction_on_Q} can be automatically ignored because we are able to interchange the order of expectation and summation and the probability measure space can be simplified to \eqref{polarity}. {\label{groups} The case of multiple groups (as defined e.g.  in \cite{biagini2020fairness} Definition 2.5) could be treated as well.} We also stress that while we a priori only look at vectors of probability measures with equal components by theoretical arguments, {\label{finitepen}we do not a priori impose the finiteness requirement in \eqref{polarity}, as this is supposed to be learned by the algorithm in the optimization. On the practical side, no particular issues arise in the approximation, as we estimate from finite samples.}
On one hand, we use the neural network $\Theta: \RR^N \to \RR^+$, parameterized by $(w_\theta, b_\theta)$ to estimate the Radon-Nikodym derivative $\frac{d Q}{d \PP}$ with respect to the physical measure. $\Theta$ takes $\mathbf{X}$ as input and generates nonnegative output with unit mean, which can be realized by using a final \textbf{Softplus}\footnote{$\text{Softplus}(x) = \log(1+\exp(x))$.} layer and dividing the outputs by their average. 




\label{estimatealpha}We will need to estimate $\alpha_B(\mathbf{Q})$, in particular for $\mathbf{Q}$ having $\sigma(\mathbf{X})$-measurable Radon-Nikodym derivative $\frac{d Q}{d \PP}=\Theta(\mathbf{X})$. 
We observe that from \eqref{eqn:alpha}, whenever $\frac{d Q}{d \PP}=\Theta(\mathbf{X})$, we have
$$\alpha_B(\bfQ):= \sup\left\{ \sum_{n=1}^N\mathbb{E}_{Q}[-Z^n]\mid \mathbf{Z}\in(\Linfty(\sigma(\mathbf{X})))^N, \Ep{U(\mathbf{Z})}\geq B\right\}$$
for $\sigma(\mathbf{X})$ the sigma algebra generated by the random vector $\mathbf{X}$.
Indeed, the ineuqlity ($\geq$) is clear, and we show ($\leq$): for every $\mathbf{Z}\in \mathbf{Z}\in(\Linfty)^N$ s.t.  $\Ep{U(\mathbf{Z})}\geq B$ we also clearly have $\widehat{\mathbf{Z}}:=\Ep{\mathbf{Z}\middle|\sigma(\mathbf{X})}\in(\Linfty(\sigma(\mathbf{X})))^N$ and $\Ep{U(\widehat{\mathbf{Z}})}=\Ep{U(\Ep{\mathbf{Z}\middle|\sigma(\mathbf{X})})}\geq\Ep{\Ep{U(\mathbf{Z})\middle|\sigma(\mathbf{X})}}=\Ep{U(\mathbf{Z})}\geq B.$ Moreover, $\sum_{n=1}^N\mathbb{E}_{Q}[-Z^n]=\sum_{n=1}^N\Ep{-Z^n\theta (\mathbf{X})}$ so that by standard manipulations $\sum_{n=1}^N\mathbb{E}_{Q}[-Z^n]=\sum_{n=1}^N\Ep{\Ep{-Z^n\middle|\sigma(\mathbf{X})}\theta (\mathbf{X})}=\sum_{n=1}^N\mathbb{E}_{Q}[-\widehat{Z}^n]$.

Based on the previous argument, to approximate $\alpha_B(\bfQ)$ we construct another neural network $\Psi:\RR^N \to \RR^N$, parameterized by $(w_\psi, b_\psi)$, taking $\mathbf{X}$ as arguments to generate random variables 
$$\mathbf{Z}:=(Z^1,\dots, Z^N)=(\Psi_1(\mathbf{X}), \dots, \Psi_N(\mathbf{X}))=:\Psi(\mathbf{X}).$$ 

Similarly to the primal problem, we have a constraint in evaluating $\alpha_B(\mathbf{Q})$ imposed by the acceptance set $\mathcal{A}$ and thus we include a penalty term when the random variable $\mathbf{Z}$ falls outside of $\mathcal{A}$. Denoting the objective function to be optimized for $\alpha_B(\mathbf{Q})$ by 
\begin{align*}
    J_\alpha (\Psi, \Theta)=  \sum_{n=1}^N \EE_\probp[-\Psi_n(\mathbf{X})\Theta(\mathbf{X}) ] - \lambda_\alpha \big(B - \EE_\probp\big[ U(\Psi(\mathbf{X}))\big]\big)^+,
\end{align*}
we write for the approximation of $\alpha_B(\bfQ)=\alpha_B(\Theta)$
$$\tilde{\alpha}_B(\Theta) = \sup_{w_\psi, b_\psi} J_\alpha (\Psi, \Theta).$$
Therefore, the objective function for dual problem is given by
\begin{align*}
    J_{\text{dual}}(\Theta) := \sum_{n=1}^N \EE_\probp\big[ -X^n\cdot \Theta(\mathbf{X}) \big] - \tilde{\alpha}_B( \Theta)
\end{align*}
and we write
$$\tilde{\rho}(\mathbf{X}) = \sup_{w_\theta, b_\theta} J_{\text{dual}}(\Theta).$$
At optimal, the penalty terms in the objective functions are almost $0$ and thus the estimations $\tilde{\alpha}_B$ and $\tilde{\rho}$ can approximate the true $\alpha_B$ and $\rho$ very well in the dual Problem \ref{problem:dual}. Since the training process based on $J_{\text{dual}}$ does not depend on the parameters $w_\psi, b_\psi$ of neural network $\Psi$. Our model is trained in the same fashion as training GANs where two neural networks contesting with each other. In practice, we follow the same convention in Algorithm 1 in \cite{goodfellow2020generative} where we calculate $\tilde{\alpha}_B$ with Monte Carlo in each step and compute gradients for both neural networks $\Psi,\Theta$ at the same time. We provide a detailed description\footnote{We use $\hat{\EE}$ to represent empirical expectation in this algorithm.} in Algorithm \ref{algo:dual}. 

\begin{algorithm}
\caption{Dual problem.}
\label{algo:dual}
\begin{algorithmic}
\REQUIRE Data $\{\mathbf{X}(\omega_i)\}_{i=1}^{\text{batch}}$, neural nets $\Psi, \Theta$, function $U$ and hyperparameters $\lambda_\alpha$, $B$, learning rates $\gamma_1, \gamma_2$, Epochs

\FOR{$e=1$ to Epochs}
\STATE compute empirical estimation $\hat{J}_\alpha$ of $J_{\alpha}(\Psi, \Theta)$  by Monte Carlo

\STATE compute empirical estimation $\hat{J}_{\text{dual}}$ of $J_{\text{dual}}(\Psi, \Theta)$  by Monte Carlo based on
$$\hat{J}_{\text{dual}}=\sum_{n=1}^N \hat{\EE}\big[ -X^n\cdot \Theta(\mathbf{X}) \big] -\hat{J}_\alpha $$

\STATE compute gradients $\nabla_{w_\psi}\hat{J}_{\alpha}, \nabla_{b_\psi}\hat{J}_{\alpha}$ and $\nabla_{w_\theta}\hat{J}_{\text{dual}}, \nabla_{b_\theta}\hat{J}_{\text{dual}}$

\STATE update $\Psi$: $$w_\psi=w_\psi - \gamma_1\nabla_{w_\psi}\hat{J}_{\alpha},\quad b_\psi=b_\psi - \gamma_1\nabla_{b_\psi}\hat{J}_{\alpha}$$

\STATE update $\Theta$: $$w_\theta=w_\theta + \gamma_2\nabla_{w_\theta}\hat{J}_{\text{dual}}, \quad b_\theta=b_\theta + \gamma_2\nabla_{b_\theta}\hat{J}_{\text{dual}}$$

\ENDFOR
\STATE compute empirical: $\hat{\alpha}_B(\mathbf{Q})=\hat{J}_\alpha(\Psi, \Theta)$ and $\hat{\rho}=\hat{J}_{\text{dual}}(\Psi, \Theta)$
\ENSURE Neural nets $\Psi, \Theta, \hat{\alpha}_B(\mathbf{Q}), \hat{\rho}$
\end{algorithmic}
\end{algorithm}


In this section, we construct two deep learning algorithms that can approximate the risk allocations $\varphi(\mathbf{X})$, the overall risk $\tilde{\rho}$ and Radon-Nikodym derivative $\Theta(\mathbf{X})$. We conclude this section by providing the approximation to \textit{fair} risk allocations \eqref{eqn:IndividualRiskAllocation} of each financial institution which combines Algorithm \ref{algo:primal}-\ref{algo:dual}, i.e. 

$$\rho^n(\mathbf{X})=\EE\big[ \varphi_n(\mathbf{X}) \Theta(\mathbf{X}) \big],\quad \forall n=1,\dots,N.$$

\section{Experiments}
\label{section4:exp}
We justify Algorithm \ref{algo:primal}-\ref{algo:dual} with three experiments. First, in section~\ref{sec:PairedNormalgroupExp}, we assume Gaussian distributed risk factors 
and paired exponential utility function where we can compare experiment results with explicit results shown in Theorem \ref{thmexplicitformulas}. Then we still take the same risk factors but apply other choice of the utility function and present numerical results as solutions in section~\ref{sec:OtherUNormalgroupExp}. In the last experiment of section~\ref{sec:PairedUBetagroupExp}, we create correlated beta distributed risk factors and search for the numerical solutions under exponential utility functions.

Both of our training and testing data consist of 50000 samples. Each sample is a realization of the $N$ ($=10$) dimensional vector $\mathbf{X}$, representing risk factors of $10$ positively correlated financial institutions who are jointly normally distributed (we actually use a very large truncation to ensure $\bfX\in(\Linfty)^N$ as in the theoretical part of the paper). We choose individual exponential utility functions 
\begin{equation}\label{eqn:individualexp}
u_n(x) = 1-e^{-\alpha_n x}\quad \text{for}\quad n=1,\ldots,10,
 \end{equation}
 where $\{\alpha_n, n=1,\ldots,10\}$ are the utility parameters for all financial institutions and they are drawn randomly from a uniform distribution between 1 and 3. As a result, the values of $\{\alpha_n, n=1,\ldots,10\}$ in the following experiments are 
 \begin{equation}\label{alphaValues}
     [1.11,\, 1.20,\, 1.36,\, 1.89,\, 1.94, \, 2.04,\, 2.27,\, 2.33,\, 2.63,\, 2.99].
 \end{equation}
 We select $B < 0$ and  we use Stochastic Gradient Descent as our deep learning optimizer for all experiments.  At each epoch, we pick a mini-batch of training samples whose batch size is 1000 and calculate gradient descents to update our parameters.
{\label{refproject} We mention that other approaches like a gradient projection method  could in principle be applied to solve our constrained optimization problem. Furthermore, a projection step generalizing to an infinite dimensional\slash functional setup the one in  \cite{feinstein2022deep} Section 3.1 could ensure feasibility and (desirable) overestimation of the overall risk $\rho_B(\bfX)$ (regarding overestimation, one could also replace $B$ with $B+\epsilon$ for a small constant $\epsilon$). At the same time, since we  need to apply a Monte Carlo method, we still expect some error which might interfere with actual overestimation. We leave these interesting issues to further research.}  For the other detailed descriptions of hyperparameters and hidden layer sizes, please see Appendix \ref{appendix:expdetails}.
 
To assess the quality of our algorithm, we consider the following evaluation metrics:
\begin{itemize}
	\item \textbf{Absolute difference.} Absolute value of the difference between estimation and theoretical solution.
    \item \textbf{Overall relative difference (ORD).} Let $\hat{E}$ be an estimation of $E$, we define the ORD by
    $$R(\hat{E}, E) = \frac{\|\hat{E} - E\|_1}{\|E\|_1}$$
    with $\|\cdot\|_1$ as the $l_1$ metric when $E$ is a deterministic vector, and $\|\cdot\|_1$ as the $L_1$ metric when $E$ is a random variable.
\end{itemize}
Small values of evaluation metrics imply better performance.

\subsection{Paired Exponential Utility}
\label{sec:PairedNormalgroupExp}

With the utility function introduced in Theorem \ref{thmexplicitformulas}, we compare with the optimal solutions \eqref{eqn:optrho}-\eqref{eqn:optQ} generated by Monte Carlo method to show the accuracy of our proposed Algorithms \ref{algo:primal}-\ref{algo:dual}.

\paragraph{Evaluation.} First, we show performance of numerical results for the estimated overall risk allocation ${\rho}$ and the estimated penalty ${\alpha_B}$ in Table~\ref{tab:RhoAlpha}, along with the expected optimal results. We run our experiments 10 times and show the average values of our estimations along with the standard deviation. The absolute differences (Abs. Difference) are also small which indicates the estimation for the overall risk is quite accurate. 
\begin{table}
\centering
  \caption{{Values of $\rho$ and $\alpha_B$}}
  \label{tab:RhoAlpha}
  \begin{tabular}{lcc}
    \toprule
    &$\rho$ &$\alpha_B$ \\
    \midrule
    Theoretical & -14.544 & 0.688\\
    Estimated (mean) & -14.511  & 0.702 \\  
    Estimated (standard deviation) & 0.0327  & 0.0142\\  
  \midrule
  \textbf{Abs. Difference}&0.033 &0.014
  \\
  \bottomrule
\end{tabular}
\end{table}

To assess the goodness of estimation of Radon-Nikodym derivatives of the optimizer, we use the overall relative difference to measure how one derivative function is different from the reference derivative function, i.e. 
\[
R(\widehat{\frac{\mathrm{d}Q_\bfX}{\mathrm{d}\mathbb{P}}},{\frac{\mathrm{d}Q_\bfX}{\mathrm{d}\mathbb{P}}}).
\]
It turns out the ORD of estimated Radon-Nikodym derivative is {\bf 5.2\%} and the behavior of the estimated measure derivative $\widehat{\frac{\mathrm{d}Q_\bfX}{\mathrm{d}\mathbb{P}}}$ 
in terms of the sum of risk factors $S(\omega)$ for each scenario $\omega$ is shown in Figure~\ref{fig:Q-S}. Both the results  show it fits optimal Radon-Nikodym derivative very well.
\begin{figure}[h]
  \centering
  \includegraphics[width=0.7\linewidth]{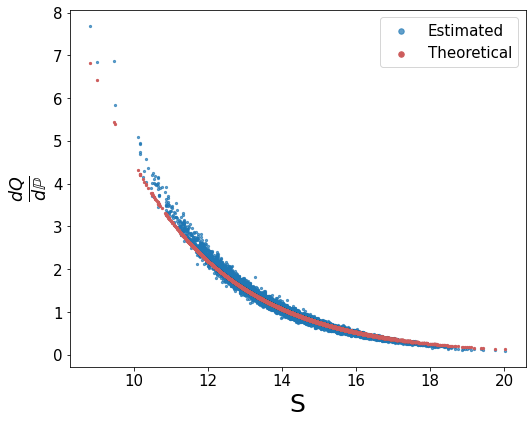}
  \caption{Behavior of $\widehat{\frac{\mathrm{d}Q_\bfX}{\mathrm{d}\mathbb{P}}}(\omega)$ against $S(\omega)$.}
  \label{fig:Q-S}
\end{figure}

Table~\ref{tab:EQY} shows the estimated \textit{fair} risk allocations for all institutions
\[
\widehat{\mathbb{E}_{Q_\bfX}[Y_\bfX]}\,:=\, \left(\widehat{\mathbb{E}_{Q_\bfX}[Y_\bfX^n]}\right)_{n=1,\ldots,10} = \left(\mathbb{E}_\probp\left[\widehat{Y_\bfX^n}\cdot \widehat{\frac{dQ_\bfX}{d\mathbb{P}}}\right]\right)_{n=1,\ldots,10}
\]
and their theoretical optimal values. The overall relative difference (ORD) for them is defined as 
\[
R\,(\widehat{\mathbb{E}_{Q_\bfX}[Y_\bfX]},{\mathbb{E}_{Q_\bfX}[Y_\bfX]}).
\]
The ORD is {\bf 4.80\%} which shows {great approximation of our algorithms to the theoretical optimal \textit{fair} allocation.}
\begin{table*}
  \caption{Estimated and theoretical optimal values of ${\mathbb{E}_{Q_\bfX}[Y_\bfX^n]}$ in one-group.}
  \label{tab:EQY}
  \centering
  \begin{tabular}{cccccccccccc}
    \toprule
     $n$&1 & 2 &3&4&5&6\\
    \midrule
    Theoretical& -0.98 &-0.73 &-0.47 &-0.65 &-0.93 &-1.81
     \\
     Estimated & -0.93 &-0.68 &-0.39 &-0.56 &-0.81 &-1.73  \\
    \bottomrule
    $n$&7&8&9&10& \textbf{ORD}\\
    \midrule
    Theoretical &-1.96 &-2.22 &-2.36 &-2.44& \multirow{2}{*}{4.80\%} 
     \\
     Estimated & -1.89 &-2.19 &-2.35 &-2.33\\
    \bottomrule
  \end{tabular}
\end{table*}


\subsubsection{Sanity Check of Radon-Nikodym Derivatives}
Based on the optimal solution and the observation from Figure~\ref{fig:Q-S}, we find the Radon-Nikodym derivative ${\frac{\mathrm{d}Q_\bfX}{\mathrm{d}\mathbb{P}}}$ is only negatively exponentially correlated with the sum of the risk factors $S = \sum_{n=1}^N X^n$. Thus here we would like to further check the dependency of ${\frac{\mathrm{d}Q_\bfX}{\mathrm{d}\mathbb{P}}}$ on $S$ by fixing the risk factor sum $S = 15$. We still follow the same setup of the distribution and utility as before and show the performance in Table~\ref{tab:RhoAlpha-fixSum} and Figure~\ref{fig:Q-S-fixSum}.
\begin{table}
\centering
  \caption{Values of $\rho$ and $\alpha_B$}
  \label{tab:RhoAlpha-fixSum}
  \begin{tabular}{lcc}
    \toprule
    &$\rho$ &$\alpha_B$ \\
    \midrule
   Theoretical & -15.348 & 0.348\\
   Estimated & -15.347 & 0.347\\
  \midrule
  \textbf{Abs. Difference}&0.001 &0.001
  \\
  \bottomrule
\end{tabular}
\end{table}

\begin{figure}[h]
  \centering
  \includegraphics[width=0.7\linewidth]{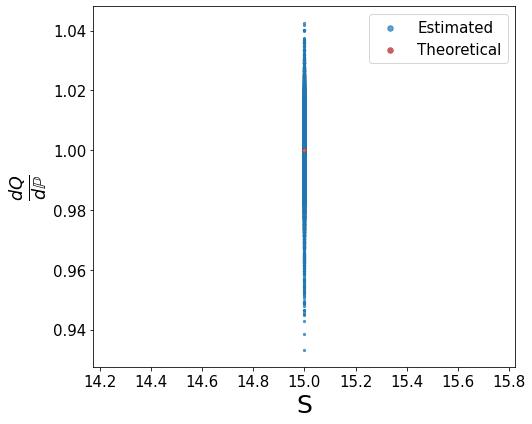}
  \caption{Behavior of $\widehat{\frac{\mathrm{d}Q_\bfX}{\mathrm{d}\mathbb{P}}}(\omega)$ against $S(\omega)$.}
  \label{fig:Q-S-fixSum}
\end{figure}

The estimations of the overall risk allocation $\rho$ and the term $\alpha_B$ are close to the theoretical values. Their values are consistent with the fact that the sum of the overall risk and $\alpha_B$ term should be equal to the sum of risk factors, according to the dual problem. In the figure, the estimation of the Radon-Nikodym derivative ${\frac{\mathrm{d}Q_\bfX}{\mathrm{d}\mathbb{P}}}$ is around 1 for all the sample, and the overall relative difference is just {\bf 0.92\%}. Our algorithm can correctly capture the relationship between the Radon-Nikodym derivative and the sum of risk factors.

\subsection{Other Multivariate Utility}
\label{sec:OtherUNormalgroupExp}
Another explicit example of the multivariate utility functions satisfying the necessary condition in Assumption~\ref{safeassumtpion} is given by 
\begin{equation}
    U(x_1,\ldots,x_N) = \sum\limits_{n=1}^N u_n(x^n)+
    \left(1-\exp{(-p\,\sum\limits_{n=1}^N\beta_nx^n)}\right),
\end{equation}
with $\beta_n\geq 0$ for $n=1,\ldots,N$ and $p>1$. Here $u_1,\ldots,u_N$ are exponential utility functions defined in \ref{eqn:individualexp}, i.e. $u_n(x)=1-e^{-\alpha_n x}$ for all $n$,
and the utility parameters are given in \ref{alphaValues}.
There is no explicit solution under this setup and we will implement it with our algorithms assuming correlated Gaussian risk factors $\mathbf{X}$ for $10$ financial institutions.

Following the same spirit of choosing $\{\alpha_n,n=1,\ldots,10\}$, we draw the coefficients $\{\beta_n,n=1,\ldots,10\}$ from a uniform distribution between 0 and 1 and thus their values are
\begin{equation}\label{valuesof betas}
    [0.65,\, 0.96,\, 0.04,\, 0.72,\, 0.77,\, 0.15,\, 0.97,\, 0.60,\, 0.81,\, 0.89].
\end{equation}
Taking $p=2$, the estimated overall risk allocation $\hat{\rho} = -14.383$; and the estimated \textit{fair} risk allocations for all institutions $\left\{\mathbb{E}_{Q_\bfX}[Y_\bfX^n], n=1,\ldots,10\right\}$ are
\begin{equation*}
 [-0.95 ,-0.79 ,-0.5,  -0.63, -0.76 ,-1.73 ,-1.91, -2.27, -2.43, -2.4].
\end{equation*}
Note that the sum of the estimated \textit{fair} risk allocations is $-14.37$ and it is close to the estimated overall risk allocation above. It shows the consistency of our results in this case without explicit solutions.
We also find that the estimated Radon-Nikodym derivative is a function of the sum of the risk factors $S$, for any values of $p$, as shown in Figure~\ref{fig:Q-S-otherU}. This is consistent with item 5 in Theorem~\ref{thm:recap} stating that $\rn{Q_\bfX}{\probp}$ are (essentially) $\sigma(S)$-measurable.
\begin{figure}[h]
  \centering
\includegraphics[width=0.7\linewidth]{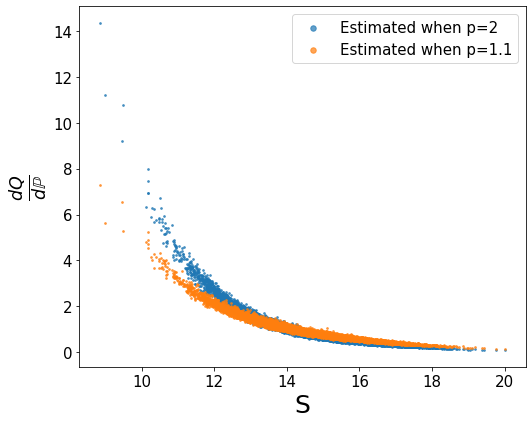}
  \caption{Behavior of $\widehat{\frac{\mathrm{d}Q_\bfX}{\mathrm{d}\mathbb{P}}}(\omega)$ against $S(\omega)$.}
  \label{fig:Q-S-otherU}
\end{figure}

\subsection{Alternative Risk Distribution}
\label{sec:PairedUBetagroupExp}
In section~\ref{sec:PairedNormalgroupExp}, we assumed a correlated Gaussian distribution for the risk factors. In this section we consider a correlated beta distribution instead. For a $N$-player system, we generate $N+1$ identically and independently distributed random variables $(Z_1,\ldots,Z_{N+1})$ following \textbf{Beta(2,5)} distribution. Then we can obtain correlated risk factors by taking $(X_1,\ldots,X_N) = (Z_1+Z_{N+1},\ldots,Z_N+Z_{N+1})$. Assuming the paired exponential utility case in section~\ref{sec:PairedNormalgroupExp}, we have explicit solutions as shown in \eqref{eqn:optrho}-\eqref{eqn:optQ}. We generate the optimal solutions by Monte Carlo method and show the accuracy of our estimation.

\paragraph{Evaluation.} First, the estimations of overall risk allocation of ${\rho}$ and the estimated penalty ${\alpha_B}$ are shown in Table~\ref{tab:RhoAlpha-beta}. We repeat the experiments for 10 times and show the estimations along with the standard deviations. The overall relative difference of the Radon-Nikodym derivatives is {\bf 3.03\%} and its behavior with respect to the sum of the risk factors is shown in Figure~\ref{fig:Q-S-beta}. Table~\ref{tab:EQY-beta} compares the estimations and optimal solutions of \textit{fair} risk allocations for all institutions. The overall relative difference is {\bf 5.34\%}.
From all the results above, we can conclude our estimations for all quantities fit the theoretical solutions very well.

\begin{table}
\centering
  \caption{{Values of $\rho$ and $\alpha_B$}}
  \label{tab:RhoAlpha-beta}
  \begin{tabular}{lcc}
    \toprule
    &$\rho$ &$\alpha_B$ \\
    \midrule
   Theoretical & -5.600 & 0.749\\
   Estimated (mean) & -5.565 & 0.743\\
   Estimated (standard deviation) & 0.0345 & 0.0065\\
  \midrule
  \textbf{Abs. Difference}&0.035 &0.006
  \\
  \bottomrule
\end{tabular}
\end{table}

\begin{figure}[h]
  \centering
  \includegraphics[width=0.7\linewidth]{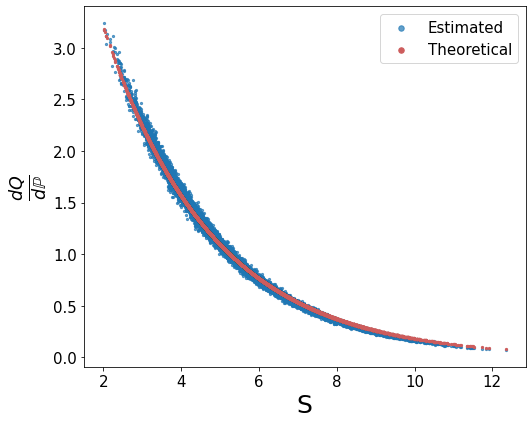}
  \caption{Behavior of $\widehat{\frac{\mathrm{d}Q_\bfX}{\mathrm{d}\mathbb{P}}}(\omega)$ against $S(\omega)$.}
  \label{fig:Q-S-beta}
\end{figure}

\begin{table*}
  \caption{Estimated and theoretical optimal values of ${\mathbb{E}_{Q_\bfX}[Y_\bfX^n]}$ in one-group.}
  \label{tab:EQY-beta}
  \centering
  \begin{tabular}{cccccccccccc}
    \toprule
     $n$&1 & 2 &3&4&5&6&\\
    \midrule
    Theoretical& -0.99& -0.89 &-0.75& -0.5 & -0.49& -0.46
     \\
     Estimated & -0.95& -0.84& -0.72& -0.48& -0.46 &-0.44\\
    \bottomrule
    $n$&7&8&9&10& \textbf{ORD}\\
    \midrule
    Theoretical& -0.42& -0.41& -0.37& -0.34& \multirow{2}{*}{5.34\%}
     \\
     Estimated & -0.39& -0.38& -0.33 &-0.3 \\ 
    \bottomrule
  \end{tabular}
\end{table*}

\newpage

\appendix
\section{Appendix}
\subsection{Proof of Theorem \ref{thmexplicitformulas}}\label{appendix:expsol}

Recall the form \eqref{utilalternative} for $U$, where $\alpha
_{1},\dots ,\alpha _{N}>0$ are fixed constants. Recall also the notation in \eqref{defconstants} and set additionally $A_{j}:=\frac{1}{\alpha_j}\log\left(\frac{1}{\alpha_j}\right)$.

To begin with, we need the following auxiliary result.
\begin{lemma}
\label{lemmacomputeV}
Take $U(\cdot)$ as in \eqref{utilalternative}. 
Then $V(\bfw):=\sup_{\bfx\in\R^N}\left(U(\bfx)-\sum_{j=1}^Nx^jw^j\right), \bfw\in(0,+\infty)^N$ satisfies: 
\begin{align*}
V(w^1,\dots,w^N)&=\frac{N^2}{2}+\sum_{j=1}^N\left(\frac{w^j}{\alpha_j}\log\left(\frac{w^j}{\alpha_j}\right)\right)+\\
&-\frac{1}{2}\left[\sum_{j=1}^N\frac{w^j}{\alpha_j}+\left(\sum_{j=1}^N\frac{w^j}{\alpha_j}\right)\log\left(\sum_{j=1}^N\frac{w^j}{\alpha_j}\right)\right]\,.
\end{align*}
and for any $z\in (0,+\infty)$ s.t. $(z,\dots,z)\in\mathrm{dom}(V)$.
\begin{align}
\label{valueVz}
&V(z,\dots, z)=\frac{N^2}{2}+\beta z\log(z)+z\Gamma-\frac{1}{2}\left(z\beta+\beta z\log(z)+z\beta\log(\beta)\right)\\
\label{valuegradV}
&\frac{\partial V}{\partial w^j}(z,\dots, z)=\frac{1}{\alpha_j}\log\left(\frac{z}{\alpha_j}\right)-\frac{1}{2\alpha_j}\log\left(z\beta\right)\\
\label{valuediv}
&V(z,\dots, z)-z\sum_{j=1}^N\frac{\partial V}{\partial w^j}(z,\dots, z)=\frac{N^2}{2}-\frac{\beta}{2}z
\end{align}
\end{lemma}

\begin{proof}
Given $V(\bfw) = \sup_{\bfx\in \mathbb{R}^N}\left(U(\bfx)-\sum_{j=1}^Nx^jw^j\right)$,
\begin{align}
    &\frac{\partial V}{\partial x^j}(\bfw) = \frac{\partial U}{\partial x^j}(\bfx)-w^j=0\quad \Longrightarrow\quad \text{the optimal choice }\bfx^* \text{ satisfies}:\nonumber
    \\
    &w^j = \alpha_je^{-\alpha_jx^{j,*}}\sum\limits_{j=1}^Ne^{-\alpha_jx^{j,*}} = h\alpha_je^{-\alpha_jx^{j,*}}\quad \text{where}\quad h = \sum\limits_{j=1}^Ne^{-\alpha_jx^{j,*}},\label{eqn:w}
    \\
    &\frac{\partial^2 V}{\partial (x^j)^2}(\bfw) = \frac{\partial^2 U}{\partial (x^j)^2}(x) = -\alpha_j^2e^{-\alpha_jx^{j,*}}\sum\limits_{j=1}^N e^{-\alpha_jx^{j,*}}-\alpha_j^2e^{-\alpha_jx^{j,*}}e^{-\alpha_jx^{j,*}} < 0.\nonumber
\end{align}
Meanwhile, \eqref{eqn:w} also implies
\begin{align*}
    &he^{-\alpha_jx^{j,*}} = \frac{w^j}{\alpha_j};\quad h = \left(\sum\limits_{j=1}^N\frac{w^j}{\alpha_j}\right)^{\frac{1}{2}};\quad
    -\alpha_jx^{j,*} = \log\frac{w^j}{\alpha_j}-\frac{1}{2}\log\left(\sum\limits_{j=1}^N\frac{w^j}{\alpha_j}\right).
\end{align*}
So under the maximizer $x^*$,
\begin{align*}
    V(\bfw) &= \sup_{x\in \mathbb{R}^N}\left(U(x)-\sum_{j=1}^Nx^jw^j\right) = U(x^*)-\sum_{j=1}^Nx^{j,*}w^j
    \\
    &=\frac{N^2}{2}-\frac{1}{2}\left(\sum\limits_{j=1}^Ne^{-\alpha_jx^{j,*}}\right)^2-\sum_{j=1}^Nx^{j,*}w^j
    \\
    &=\frac{N^2}{2}-\frac{1}{2}\sum\limits_{j=1}^N\frac{w^j}{\alpha_j}-\sum_{j=1}^Nx^{j,*}\cdot h\alpha_je^{-\alpha_jx^{j,*}}
    \\
    &=\frac{N^2}{2}-\frac{1}{2}\sum\limits_{j=1}^N\frac{w^j}{\alpha_j}+\sum_{j=1}^N\left(\log\frac{w^j}{\alpha_j}-\frac{1}{2}\log\left(\sum\limits_{j=1}^N\frac{w^j}{\alpha_j}\right)\right)\frac{w^j}{\alpha_j}
    \\
    &=\frac{N^2}{2}+\sum\limits_{j=1}^N\,\frac{w^j}{\alpha_j}\log\frac{w^j}{\alpha_j}-\frac{1}{2}\left[\sum\limits_{j=1}^N\frac{w^j}{\alpha_j}+\sum\limits_{j=1}^N\frac{w^j}{\alpha_j}\log\left(\sum\limits_{j=1}^N\frac{w^j}{\alpha_j}\right) \right].
\end{align*}
By direct computation, the derivative can be obtained 
\[
\frac{\partial V}{\partial w^j}(\bfw) = \frac{1}{\alpha_j}\log\frac{w^j}{\alpha_j}-\frac{1}{2\alpha_j}\log\left(\sum\limits_{j=1}^N\frac{w^j}{\alpha_j}\right).
\]
\end{proof}

\begin{proof}[Proof of Theorem \ref{thmexplicitformulas}]
In the proof, to avoid heavy notation, we drop the specification of the dependence of $\bfX$ of the optima, and use a "hat" superscript to denote them: $\bfY_\bfX=\widehat{\bfY},\bfQ_\bfX=\widehat{\bfQ}
$.

We first show that all the requirements on $U$ in Assumption \ref{safeassumtpion} are met. 
We observe that for $\varepsilon\geq0$
\begin{align*}
 U(\bfx-\varepsilon\bfone)&= \frac{N^2}{2}-\frac{1}{2}\left(\sum_{n=1}^N\exp\left(\alpha_n\varepsilon\right)\exp\left(-\alpha_nx^n\right)\right)^2\\
 &\geq  \frac{N^2}{2}-\frac{1}{2}\left(\sum_{n=1}^N\exp\left(\max_n\alpha_n\varepsilon\right)\exp\left(-\alpha_nx^n\right)\right)^2\\
 &=\frac{N^2}{2}-\exp\left(2\max_n\alpha_n\varepsilon\right)\frac{1}{2}\left(\sum_{n=1}^N\exp\left(-\alpha_nx^n\right)\right)^2\\
 &=\left(1-\exp\left(2\max_n\alpha_n\varepsilon\right)\right)\frac{N^2}{2}+\exp\left(2\max_n\alpha_n\varepsilon\right)U(x)
\end{align*}

so that, since the negative part function is nonincreasing and subadditive  \[\left(U(\bfx-\varepsilon\bfone)\right)^-\leq \left(\exp\left(2\max_n\alpha_n\varepsilon\right)-1\right)\frac{N^2}{2}+\exp\left(2\max_n\alpha_n\varepsilon\right)\left(U(\bfx)\right)^-.\]
Hence, if $\bfX\in (\Lone)^{N}$ satisfies $(U(\bfX))^{-}\in
\Lone$ then there exists $\delta >0$
s.t. $(U(\bfX-\varepsilon \bfone))^{-}\in \Lone$ for all $ 0\leq \varepsilon <\delta$.
Moreover for some constant $\theta\in\R$
\begin{align*}
  \Lambda(-\lambda\abs{\bfx})&=\frac{1}{2}\sum\limits_{n,m=1; n\neq m}^N(1-e^{\lambda(\alpha_n \abs{x^n}
+\alpha_m\abs{x^m})})\\
&\geq \theta-\frac{1}{2}\sum\limits_{n,m=1; n\neq m}^N\frac12\left(e^{2\lambda\alpha_n \abs{x^n}
}+e^{2\lambda\alpha_m \abs{x^m}
}\right) \\
\end{align*}
If for $\bfX\in (L^{0}\left( (\Omega ,\mathcal{F},{%
\mathbb{P}});[-\infty ,+\infty ]\right) )^{N}$ there exist $\lambda
_{1},\dots ,\lambda _{N}>0$ such that $\Ep{
u_{j}(-\lambda _{j}\left\vert X^{j}\right\vert )} >-\infty $ (which implies $\bfX$ is a.s. finite valued in particular), then
for $\alpha=\frac12\min_j\lambda_j>0$ we have $\Ep{
u_{j}(-2\alpha\left\vert X^{j}\right\vert )}>-\infty$ for every $j$, and  $\Ep{ \Lambda
(-\alpha \left\vert \bfX\right\vert )} >-\infty $ by the computation above.

We can now move to proving the validity of the formulas provided in the statement.
Observe that $\rn{\widehat{\probq}}{\probp}\in\Linfty$ and that  $ \widehat{\lambda}(B)\in \R$ is well defined and $\widehat{\lambda}(B)<0$ since $B<\sup_\bfx U(\bfx)=\frac{N^2}{2}$. Moreover  we have \[\widehat{\bfY}=-\bfX-\nabla V\left(\widehat{\lambda}(B)\rn{\widehat{\probq}}{\probp}\mathbf{1}\right)\,,\quad \widehat{\bfY}\in (\Linfty)^N\,\quad \sum_{j=1}^N\widehat{Y}^j=d(\bfX)-\Gamma.\]

Indeed, using \eqref{valuegradV} 
\begin{align*}
&\phantom{=}-X^j-\frac{\partial V}{\partial w^j}\left(\widehat{\lambda}\rn{\widehat{\probq}}{\probp}\mathbf{1}\right)=-X^j-\left[\frac{1}{\alpha_j}\log\left(\frac{1}{\alpha_j}\widehat{\lambda}(B)\rn{\widehat{\probq}}{\probp}\right)-\frac{1}{2\alpha_j}\log\left(\beta\widehat{\lambda}\rn{\widehat{\probq}}{\probp}\right)\right]\\
&=-X^j+\frac{1}{2\alpha_j}\left[\log(\beta)-\log\left(\widehat{\lambda}\rn{\widehat{\probq}}{\probp}\right)\right]-\frac{1}{\alpha_j}\log\left(\frac{1}{\alpha_j}\right)\\
&=-X^j+\frac{1}{2\alpha_j}\left[\log(\beta)+\log\left(\frac{1}{\widehat{\lambda}(B)}\Ep{\exp\left(-\frac{2S}{\beta}\right)}\right)-\log\left(\exp\left(-\frac{2S}{\beta}\right)\right)\right]+\\
&-\frac{1}{\alpha_j}\log\left(\frac{1}{\alpha_j}\right)
\\
&=-X^j+\frac{1}{2\alpha_j}\left[\log\left(\frac{\beta}{\widehat{\lambda}(B)}\Ep{\exp\left(-\frac{2S}{\beta}\right)}\right)+\frac{2}{\beta}S\right]-\frac{1}{\alpha_j}\log\left(\frac{1}{\alpha_j}\right)\\
&=-X^j+\frac{1}{\beta\alpha_j}\left[\frac{\beta}{2}\log\left(\frac{\beta}{\widehat{\lambda}(B)}\Ep{\exp\left(-\frac{2S}{\beta}\right)}\right)+S\right]-\frac{1}{\alpha_j}\log\left(\frac{1}{\alpha_j}\right)\\
&=-X^j+\frac{1}{\beta\alpha_j}\left[d(\bfX)+S\right]-\frac{1}{\alpha_j}\log\left(\frac{1}{\alpha_j}\right)\,.
\end{align*}
$\widehat{\bfY}\in (\Linfty)^N$ is evident, and so is $\sum_{j=1}^N\widehat{Y}^j=d(\bfX)-\Gamma$.

  Furthermore 
  \begin{equation}
  \label{expressionB}
  \Ep{U(\bfX+\widehat{\bfY})}=\Ep{V\left(\widehat{\lambda}(B)\rn{\widehat{\probq}}{\probp}\mathbf{1}\right)-\widehat{\lambda}(B)\rn{\widehat{\probq}}{\probp}\sum_{j=1}^N\frac{\partial V}{\partial w^j}\left(\widehat{\lambda}(B)\rn{\widehat{\probq}}{\probp}\mathbf{1}\right)}=B
  \end{equation}

 where the first equality derives from the well known identity $U(-\nabla V(\bfw))=V(y)-\sum_{j=1}^Nw^j\frac{\partial V}{\partial w^j}(\bfw)$, for $\bfw\in\mathrm{int}(\mathrm{dom}(V))$ (see \cite{Rocka}   V.\S 26), and the second equality is checked by direct computation: 
by \eqref{valuediv}  we have \[V\left(\widehat{\lambda}\rn{\widehat{\probq}}{\probp}\mathbf{1}\right)-\widehat{\lambda}\rn{\widehat{\probq}}{\probp}\sum_{j=1}^N\frac{\partial V}{\partial w^j}\left(\widehat{\lambda}\rn{\widehat{\probq}}{\probp}\mathbf{1}\right)=-\frac{\beta}{2}\widehat{\lambda}\rn{\widehat{\probq}}{\probp}+\frac{N^2}{2}.\] and we can use  $\Ep{\rn{\widehat{\probq}}{\probp}}=1$.
 
 Then $\widehat{\bfY}$ is admissible for $\rho_B(\bfX)$, and $\rho_B(\bfX)\leq \sum_{j=1}^N\widehat{Y}^j=d(\bfX)-\Gamma$. 
 Also, by Fenchel inequality $\alpha_B(\widehat{\probq}\mathbf{1})\leq \frac{1}{\widehat{\lambda}(B)}\left(-B+\Ep{V\left(\widehat{\lambda}(B)\rn{\widehat{\probq}}{\probp}\mathbf{1}\right)}\right)$, and the latter by the previous computations equals $\sum_{j=1}^N\mathbb{E}_{\widehat{\probq}}\left[-{X^j}\right]-d(\bfX)+\Gamma\leq \mathbb{E}_{\widehat{\probq}}\left[-\sum_{j=1}^N{X^j}\right]-\rho_B(\bfX)\leq \alpha_B(\widehat{\probq}\mathbf{1})$.
 Since this implies $d(\bfX)-\Gamma=\sum_{j=1}^N\mathbb{E}_{\widehat{\probq}}\left[-{X^j}\right]-\alpha_B(\widehat{\probq}\mathbf{1})\leq \rho_B(\bfX)$ (the latter inequality coming form the dual representation of $\rcondinfty{\cdot}$), we conclude that $\rcondinfty{X}=d(X)$, 
$\widehat{Y}$ is a primal optimum and $\widehat{\probq}\mathbf{1}$ is a dual optimum. More in detail:
  \begin{align*}
  &\alpha_B(\widehat{\probq}\mathbf{1}) : =\sup\left\{ \frac{1}{\widehat{\lambda}}\sum_{j=1}^{N}\Ep{-X^j\widehat{\lambda}\rn{\widehat{\probq}}{\probp}} \mid \bfZ\in (\Linfty)^{N},\,\mathbb{E}_{\mathbb{P}}\left[ U\left( \bfZ\right) %
\right] \geq B\right\} \\
  &\stackrel{\mathrm{Fenchel}}{\leq }\sup\left\{ \frac{1}{\widehat{\lambda}}\left(\sum_{j=1}^{N}\Ep{V\left(\widehat{\lambda}\rn{\widehat{\probq}}{\probp}\mathbf{1}\right)}-\Ep{U(Z)}\right) \mid \bfZ\in (\Linfty)^{N},\,\mathbb{E}_{\mathbb{P}}\left[ U\left( \bfZ\right) %
\right] \geq B\right\}\\
&\leq \frac{1}{\widehat{\lambda}}\left(-B+\sum_{j=1}^{N}\Ep{V\left(\widehat{\lambda}\rn{\widehat{\probq}}{\probp}\mathbf{1}\right)}\right)\\
\end{align*}
\begin{align*}
&\stackrel{\eqref{expressionB}}{=} -\frac{1}{\widehat{\lambda}}\left\{\Ep{V\left(\widehat{\lambda}\rn{\widehat{\probq}}{\probp}\mathbf{1}\right)-\widehat{\lambda}\rn{\widehat{\probq}}{\probp}\sum_{j=1}^N\frac{\partial V}{\partial w^j}\left(\widehat{\lambda}\rn{\widehat{\probq}}{\probp}\mathbf{1}\right)}\right\}+
\\&+\frac{1}{\widehat{\lambda}}\sum_{j=1}^{N}\Ep{V\left(\widehat{\lambda}\rn{\widehat{\probq}}{\probp}\mathbf{1}\right)}=\Ep{\rn{\widehat{\probq}}{\probp}\sum_{j=1}^N\frac{\partial V}{\partial w^j}\left(\widehat{\lambda}\rn{\widehat{\probq}}{\probp}\mathbf{1}\right)}\\
&=\Ep{\rn{\widehat{\probq}}{\probp}\left(-\sum_{j=1}^N X^j\right)-\sum_{j=1}^N\left(-X^j-\frac{\partial V}{\partial w^j}\left(\widehat{\lambda}\rn{\widehat{\probq}}{\probp}\mathbf{1}\right)\right)}\\
 &=\sum_{j=1}^N\mathbb{E}_{\widehat{\probq}}\left[-{X^j}\right]-\mathbb{E}_{\widehat{\probq}}\left[\sum_{j=1}^N\widehat{Y}^j\right]\\
 &=\sum_{j=1}^N\mathbb{E}_{\widehat{\probq}}\left[-{X^j}\right]-\sum_{j=1}^N\widehat{Y}^j=\sum_{j=1}^N\mathbb{E}_{\widehat{\probq}}\left[-{X^j}\right]-d(\bfX)+\Gamma
  \end{align*}
  using previous computations for $\widehat{Y}$ in the last equality.

As to the explicit expression for $\alpha_B$, we observe that 

$$\rho_B(\bfX)=\rho_{\frac{2}{\beta}}\left(S-{\frac{\beta}{2}}\log\left(\frac{\beta}{\widehat{\lambda}}\right)\right)-\Gamma$$ for $\rho_{\frac{2}{\beta}}(X)=\frac{\beta}{2} \log\left(\Ep{e^{-\frac{2X}{\beta}}}\right)$ as in \cite{FollmerSchied2} Example 4.34.

From the computations in \cite{doldi2021conditional} Claim 5.8 one sees that 
\begin{align*}
&\phantom{=}\alpha_B(\probq,\dots,\probq)=\rho_B^*\left(-\rn{\probq}{\probp},\dots, -\rn{\probq}{\probp}\right)=\sup_{\bfX\in (\Linfty)^N}\left(\Eq{-\sum_{j=1}^NX^j}-\rho(X)\right)\\
&=\sup_{S\in \Linfty}\left(\Eq{-S}-\rho_{\frac{2}{\beta}}\left(S-{\frac{\beta}{2}}\log\left(\frac{\beta}{\widehat{\lambda}}\right)\right)\right)+\Gamma\\
&=\sup_{S\in \Linfty}\left(\Eq{-S+{\frac{\beta}{2}}\log\left(\frac{\beta}{\widehat{\lambda}}\right)}-\rho_{\frac{2}{\beta}}\left(S-{\frac{\beta}{2}}\log\left(\frac{\beta}{\widehat{\lambda}}\right)\right)\right)+\Gamma-{\frac{\beta}{2}}\log\left(\frac{\beta}{\widehat{\lambda}}\right)\\
&=\frac{\beta}{2}H(\probq|\probp)+\Gamma-{\frac{\beta}{2}}\log\left(\frac{\beta}{\widehat{\lambda}}\right)=\left(\Gamma-{\frac{\beta\log(\beta)}{2}}\right)+\frac{\beta}{2}\log\left({\widehat{\lambda}}\right)+\frac{\beta}{2}H(\probq|\probp)
\end{align*}
where  we used the entropy formula in the same Example 4.34. 
\end{proof}
\subsection{Auxiliary Results}
\begin{prop}
\label{propuniquedual}
Under Assumption \ref{safeassumtpion} suppose additionally that $u_1,\dots,u_N$ are bounded from above, and that $U$ is differentiable. Then there exists a unique optimum for \eqref{DynRMeqdualreprshorfall} in $\mcD$, and if $\mcC=\mcC_\R$ both $\rn{\bfQ}{\probp}$ and $\bfX+\bfY_\bfX$ are $\sigma(X^1+\dots+X^N)$-measurable.
\end{prop}
\begin{proof}
The claim on uniqueness follows from Theorem \ref{thm:msorte} together with the uniqueness result for mSORTE in \cite{doldi2022multivariate} Theorem 4.6,  provided that for $\mcB=\mcC$ we show $\mcD\subseteq\mathcal{Q}_{V,\mcB}$ (see (23) in \cite{doldi2022multivariate} for the definition of the latter, and observe that for $\mcB=\mcC$ and $\bfX\in(\Linfty)^N$ Standing Assumption II in \cite{doldi2022multivariate} holds).
This given, we can also infer the claim about measurability, from \cite{doldi2022multivariate} Proposition 4.10.
To show $\mcD\subseteq\mathcal{Q}_{V,\mcB}$, it is enough to prove that for any $\bfQ\in\mcD$ we have $\Ep{V\left(\lambda\rn{\bfQ}{\probp}\right)}<+\infty$ for some $\lambda>0$, which is now our aim.
 For all unexplained notation on Orlicz spaces we defer to \cite{doldi2021conditional} Section  2.1, and to \cite{doldi2021conditional} Section 5.2.2 for an explanation on how to naturally associate a multivariate Orlicz space to the setup of this paper.
 Under Assumption \ref{safeassumtpion}, we have $L^\Phi=\times_{n=1}^NL^{\Phi_n}$ (see \cite{doldi2021conditional} Assumption 5.12 together with  the comment immediately following it), which by \cite{doldi2022multivariate} Proposition 2.5.3 yields $K_\Phi=\times_{n=1}^NL^{\Phi^*}_n$.   By \cite{doldi2021conditional} Lemma A.11, any element of $\mcD$ belongs to $K_\Phi$, so that there exist $\lambda_1,\dots,\lambda_N>0$ for which $\Ep{\Phi_j^*\left(\lambda_j\rn{Q^j}{\probp^j}\right)}<+\infty,j=1,\dots,N$. Up to choosing the minimum of these by monotonicity of $\Phi^*_j$, we can assume that $\lambda_1=\dots=\lambda_N=\lambda$ .
 
 Now, by e.g. \cite{DoldiThesis21} Remark 2.2.8 we then have that, for $Z\geq 0$, $\Ep{v_j(Z)}<+\infty$ is equivalent to $\Ep{\Phi_j(Z)}<+\infty$ by boundedness of $u_1,\dots,u_N$, so that $\sum_{j=1}^N\Ep{v_j\left(\lambda\rn{Q^j}{\probp^j}\right)}$
 By boundedness of $\Lambda$ we have for all $\bfy\in\R^N$ 
 \[V(\bfy)=\sup_{\bfx\in\R^N}\left(U(\bfx)-\sum_{n=1}^Nx^ny^n\right)\leq \sum_{n=1}^N\sup_{x\in\R}\left(u_n(x)-xy^n\right)=\sum_{n=1}^Nv_n(y^n).\]
 Collecting our findings, we get that for any $\bfQ\in\mcD$ \[\Ep{V\left(\lambda\rn{\bfQ}{\probp^j}\right)}\leq\sum_{j=1}^N\Ep{v_j\left(\lambda\rn{Q^j}{\probp^j}\right)}<+\infty \]
 which implies $\bfQ\in\mathcal{Q}_{V,\mcB}$.
\end{proof}

\subsection{Experiment Details}\label{appendix:expdetails}
 \paragraph{Primal Problem.} In the experiment for the primal problem, we assign $(1,0.02)$ to be the initial values of the hyperparameters $(\mu,\lambda)$. There are 2 hidden layers and their dimensions are $[64,64]$. The learning rate is 0.00001 and there are 1600 epochs in total. We double the hyperparametes every 400 epochs to avoid vanishing gradient. This action makes sure the penalty terms converge to 0 while not bringing too much effect on the objective function at the beginning. This is a tuning method based on our experience and observation from the training process.

\paragraph{Dual Problem.} In the experiment for the dual problem, we have two neural networks $\Psi$ and $\Theta$. The value of the hyperparameter $\lambda_\alpha$ in $\Psi$ is 0.1. We assume 2 hidden layers for both neural networks and their dimensions are $[64,64]$. The learning rates for them are both 0.001. We take 1000 epochs and the learning rates are adjusted to a half of it every 200 epochs.

\paragraph{Training Time.} Based on the choice of layers and number of epochs, the training time for the first experiment in section \ref{sec:PairedNormalgroupExp} is 13 minutes on a standard laptop. The other experiments discussed in the paper cost similar time to train.

\section*{Funding}

Jean-Pierre Fouque was supported by  NSF grants DMS-1814091 and DMS-1953035.

\bibliographystyle{abbrv}
\bibliography{reference}

\end{document}